\pdfoutput=1
\documentclass[11pt]{article}

\usepackage[preprint]{coling}
\usepackage[most]{tcolorbox}
\usepackage{times}
\usepackage{latexsym}
\usepackage{algorithm}

\usepackage{algorithmic}
\usepackage{amsthm}
\usepackage[T1]{fontenc}

\usepackage[utf8]{inputenc}
\usepackage{microtype}
\usepackage{inconsolata}
\usepackage{makecell}
\usepackage{multirow}
\usepackage{setspace}
\usepackage{booktabs}
\usepackage{threeparttable}
\usepackage{CJKutf8}
\usepackage{enumerate}
\usepackage{amsfonts,amssymb,amsmath} 
\usepackage{color}
\usepackage{graphicx}
\usepackage{tabularx}

\newtcolorbox[list inside=prompt,auto counter,number within=section]{prompt}[1][]{
    colbacktitle=black!60,
    coltitle=white,
    fontupper=\footnotesize,
    boxsep=5pt,
    left=0pt,
    right=0pt,
    top=0pt,
    bottom=0pt,
    boxrule=1pt,
    #1,
}

\newtcolorbox[auto counter,number within=chapter]{prompt2}[1][]{
  enhanced,
  breakable,
  fontupper=\footnotesize,
  fonttitle=\scshape,
  title={Definition \thetcbcounter},
  #1,
}
\usepackage[%
    framemethod=tikz,
    skipbelow=\topskip,
    skipabove=\topskip
]{mdframed}
\mdfsetup{%
    leftmargin=0pt,
    rightmargin=0pt,
    backgroundcolor=bggray,
    middlelinecolor=black,
    roundcorner=3
}

\newmdenv[
  backgroundcolor=black!05,
  linecolor=quoteborder,
  skipabove=1em,
  skipbelow=1em,
  leftline=true,
  topline=false,
  bottomline=false,
  rightline=false,
  linecolor=black!40,
  linewidth=4pt,
  font=\small,
  leftmargin=0cm
]{prompt_env}

\title{SMART-RAG: Selection using Determinantal Matrices for Augmented Retrieval}

\author{
 \textbf{Jiatao Li\textsuperscript{1,2}},
 \textbf{Xinyu Hu\textsuperscript{1}},
 \textbf{Xiaojun Wan\textsuperscript{1}},
\\
 \textsuperscript{1}Wangxuan Institute of Computer Technology, Peking University
, \\
 \textsuperscript{2}Information Management Department, Peking University
 \\
 \texttt{{leejames@stu.pku.edu.cn}},
 \texttt{\{huxinyu,wanxiaojun\}@pku.edu.cn}
}
\newtheorem{theorem}{Theorem}
\newtheorem{corollary}{Corollary}[theorem]
\begin{document}
\maketitle
\begin{abstract}
Retrieval-Augmented Generation (RAG) has greatly improved large language models (LLMs) by enabling them to generate accurate, contextually grounded responses through the integration of external information. However, conventional RAG approaches, which prioritize top-ranked documents based solely on query-context relevance, often introduce redundancy and conflicting information. This issue is particularly evident in unsupervised retrieval settings, where there are no mechanisms to effectively mitigate these problems, leading to suboptimal context selection.
To address this, we propose \textbf{S}election using \textbf{M}atrices for \textbf{A}ugmented \textbf{R}e\textbf{T}rieval (\textbf{SMART}) in question answering tasks , a fully unsupervised and training-free framework designed to optimize context selection in RAG. SMART leverages Determinantal Point Processes (DPPs) to simultaneously model relevance, diversity and conflict, ensuring the selection of potentially high-quality contexts. Experimental results across multiple datasets demonstrate that SMART significantly enhances QA performance and surpasses previous unsupervised context selection methods, showing a promising strategy for RAG.
\end{abstract}

\section{Introduction}
Recent advancements in large language models (LLMs) have significantly boosted performance in various natural language processing (NLP) tasks, especially in question answering (QA)~\cite{NEURIPS2020_1457c0d6}. Despite their fluent and coherent generation, LLMs often face issues of hallucination and lack of factual grounding~\cite{maynez-etal-2020-faithfulness,zhou-etal-2021-detecting}, resulting in inconsistent or unreliable responses—particularly problematic in domains where factual accuracy is crucial, such as law, medicine, and education.

To address these issues, Retrieval-Augmented Generation (RAG) have emerged as a powerful solution. By integrating external sources of factual evidence, RAG enhance the reliability and consistency of QA models~\cite{khattab2023demonstratesearchpredict,izacard2022atlas}. However, selecting relevant, diverse, and non-conflicting contexts remains a major challenge. Supervised RAG methods require large amounts of labeled data, which is impractical in low-resource domains~\cite{wang2023learning,dong2024don}. Even unsupervised methods often rely on pre-training or fine-tuning, which incurs significant computational costs.

One of the key challenges in RAG is selecting contexts that are not only relevant but also diverse and free of contradictions. While retrieved contexts are typically aligned with the query, they can introduce redundancy or conflicting information, especially when drawn from large datasets. For instance, one context might claim that ``Treatment X has been shown to reduce symptoms by 50\% in trials,'' while another might state, ``Treatment X is ineffective and provides no significant benefits.'' Such contradictions can confuse the QA system, leading to unreliable or incoherent responses, which is particularly problematic in high-stakes domains like education and healthcare, where misinformation can have serious consequences.

This problem occurs because retrieval methods tend to prioritize relevance at the expense of diversity, leading to the over-selection of similar contexts. Furthermore, when multiple contexts present conflicting perspectives on the same information, the coherence of the generated answers suffers. A crucial factor for producing accurate, coherent answers lies in striking a balance between relevance and diversity while minimizing contradictions.

Therefore, we introduce \textbf{S}election using \textbf{M}atrices for \textbf{A}ugmented \textbf{R}e\textbf{T}rieval in Question Answering (\textbf{SMART}), a fully unsupervised, train-free framework that directly tackles the challenges of context selection in RAG. The key innovation of SMART lies in its use of Determinantal Point Processes (DPPs) to model relevance, diversity, and conflict relations among retrieved contexts. This approach enables SMART to select highly valuable contexts while effectively addressing issues of redundancy and conflict, which often hinder the performance of traditional RAG methods. 

SMART integrates conflict modeling into context selection to enhance the coherence and factual accuracy of generated answers. By leveraging Determinantal Point Processes (DPPs), SMART balances relevance and diversity while minimizing redundancy, preventing the selection of conflicting or repetitive contexts~\cite{kulesza2012determinantal}. Cosine similarity is used to assess relevance, and Natural Language Inference (NLI) identifies and penalizes contradictions, ensuring that selected contexts are both diverse and consistent. As a fully unsupervised, train-free method, SMART efficiently adapts to various tasks without requiring large-scale labeled data or computationally intensive retraining.

Our contributions are as follows:
\begin{enumerate}
    \item We introduce a novel extension of DPP that integrates conflict modeling into context selection for RAG.
    \item We present SMART, a fully unsupervised and training-free framework that surpasses previous unsupervised RAG methods.
    \item We demonstrate the effectiveness of our SMART in improving context selection for QA by balancing relevance, diversity, and conflict, resulting in better performance across multiple datasets.
\end{enumerate}

\section{Methodology}

\begin{figure*}[t]
\centering
\includegraphics[width=1\textwidth]{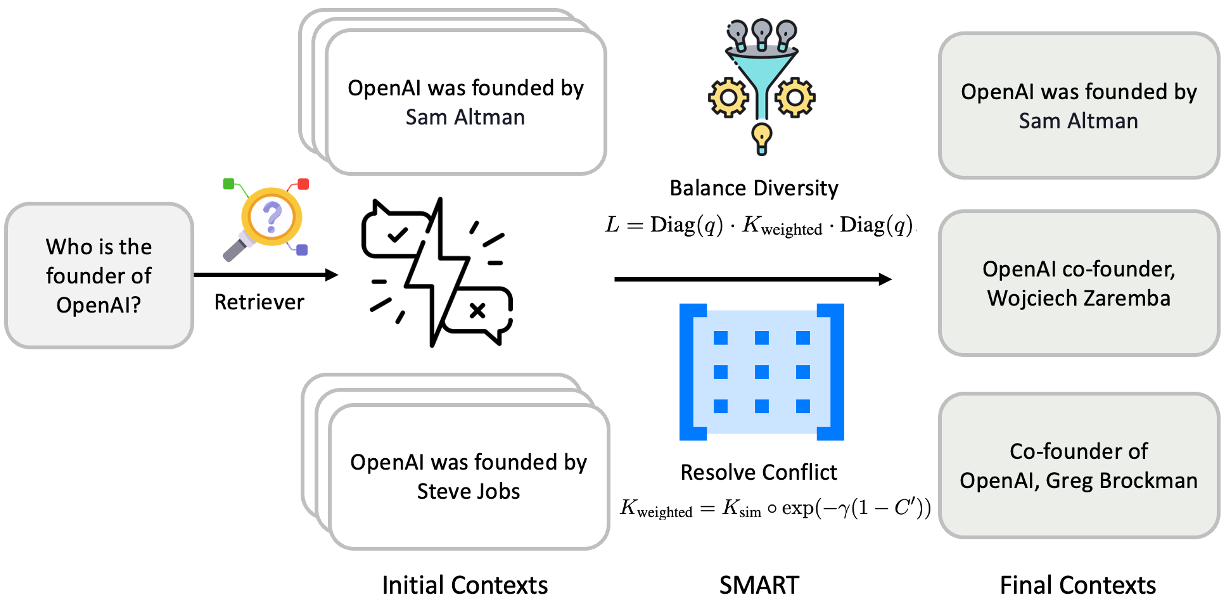} 
\caption{SMART pipeline for context selection in question answering. Given an initial query, the retriever identifies relevant documents, which are then split into sentences. A pre-ranking step is applied to narrow down the top candidate sentences. The pipeline involves two key stages: (1) Data Processing, where Query-Context (QC) and Context-Context (CC) relations are constructed, and (2) Context Selection, leveraging Determinantal Point Processes (DPP) to balance relevance and diversity while resolving conflicts between contexts. The result is a set of diverse, non-redundant, and coherent final contexts.}
\label{pipeline}
\end{figure*}

In this section, we present the SMART framework, which employs DPP to select relevant, diverse, and conflict-free contexts. By explicitly modeling conflict relations between contexts, SMART improves coherence and factual consistency in question-answering (QA) tasks, ensuring more reliable and accurate answers.

\subsection{Relation Modeling}

In SMART, each context \( c \) is segmented into individual sentences \( \{s_1, s_2, \dots, s_n\} \), after which the framework models three essential relationships: \textit{textual similarity}, \textit{conflict relations}, and \textit{query-context relevance}. 

\paragraph{Textual Similarity} To assess content overlap, we compute pairwise cosine similarity between dense vector representations of contexts. The cosine similarity between two contexts \( c_i \) and \( c_j \), denoted as \( s(c_i, c_j) \), is calculated as:
\[
s(c_i, c_j) = \frac{v_{c_i} \cdot v_{c_j}}{\|v_{c_i}\| \|v_{c_j}\|}
\]
where \( v_{c_i} \) and \( v_{c_j} \) are the embeddings of contexts \( c_i \) and \( c_j \). This similarity score contributes to the similarity matrix \( K_{\text{sim}} \), which captures the degree of textual similarity between contexts.

\paragraph{Conflict Relations}
\label{para:conflict-relations}
In addition to modeling relevance and diversity, SMART also accounts for conflict relations between contexts to avoid selecting contradictory information. For each context pair \( c_i \) and \( c_j \), we employ a Natural Language Inference (NLI) model to compute the probability of conflict. The conflict relation \( c(c_i, c_j) \) is symmetrized by averaging the directional probabilities \( P(c_i \rightarrow c_j) \) and \( P(c_j \rightarrow c_i) \):
$$
c(c_i, c_j) = \frac{1}{2} \left( P(c_i \rightarrow c_j) + P(c_j \rightarrow c_i) \right)
$$
This results in a conflict matrix \( C \), which penalizes the selection of contradictory contexts, thereby ensuring that the final set is coherent and free from conflicts. By incorporating conflict modeling, SMART enhances the factual consistency of the selected contexts, going beyond standard relevance-based methods.

\paragraph{Query-Context Relevance} In addition to modeling relationships between contexts, SMART calculates the relevance of each context to the query. Using cosine similarity between the query \( Q \) and each context \( c \), the relevance score \( r(Q, c) \) is given by:
\[
r(Q, c) = \frac{v_Q \cdot v_c}{\|v_Q\| \|v_c\|}
\]
where \( v_Q \) and \( v_c \) represent the embeddings of the query and the context, respectively. This relevance score ensures that the selected contexts are closely aligned with the query, leading to higher-quality answers.

\subsection{DPP-Based Context Selection}
Based on the context relations constructed as above, we modify and improve the Determinantal Point Process (DPP) approach to select an optimal subset of contexts. It ensures an effectively balanced selection in terms of \textit{relevance, diversity}, and \textit{coherence}. We will introduce the main method of our SMART, with the related theoretical derivations and proofs presented in Appendix~\ref{sec:proofs}.

\paragraph{Kernel Matrix Construction} 
The kernel matrix \( L \) ensures that selected contexts are relevant, diverse, and conflict-free. It is constructed as follows:
$$ L = \text{Diag}(r) \cdot K_{\text{weighted}} \cdot \text{Diag}(r) $$

where \( r \) represents the relevance scores between the query and each context, and \( K_{\text{weighted}} \) is a \textit{conflict-adjusted similarity matrix} that balances textual similarity and conflict penalties. The symmetrized conflict matrix \( C \) is computed as described in Section~\ref{para:conflict-relations}, and the weighted similarity matrix is initialized as:
\[
K_{\text{weighted}} = K_{\text{sim}} \circ \exp(-\gamma (1 - C))
\]
where \( K_{\text{sim}} \) represents the cosine similarity between contexts. The decay function \( \exp(-\gamma (1 - C)) \) adjusts the similarity based on conflict, with \( \gamma \geq 0 \) controlling how strongly conflict impacts similarity. This penalizes contexts with higher conflict scores as proved in Section~\ref{sec:c_reduce}.

\paragraph{Subset Selection via DPP}
After constructing the conflict-aware kernel matrix \( L \), DPP selects a subset \( Y_g \subseteq Y \) of contexts that optimally balances relevance, diversity, and coherence. The probability of selecting a subset \( Y_g \) is given by:

\[
P(Y_g) = \frac{\prod_{i \in Y_g} q_i^2 \cdot \det(K_{\text{weighted}, Y_g})}{\det(L + I)},
\]

where \( K_{\text{weighted}, Y_g} \) is the submatrix of \( K_{\text{weighted}} \) corresponding to the selected subset \( Y_g \). This ensures that selected contexts are both relevant and non-conflicting, mitigating the risk of incoherent or contradictory answers.

\paragraph{Efficient Subset Selection via Greedy MAP Inference}

To efficiently select an optimal subset of contexts, we use a \textit{greedy Maximum A Posteriori (MAP)} inference algorithm, iteratively selecting the context \( c_j \) that maximizes the marginal gain in the determinant:
$$ c_j = \arg\max_{i \in Y \setminus Y_g} \left( \log \det(L_{Y_g \cup \{i\}}) - \log \det(L_{Y_g}) \right) $$
This ensures that each selected context enhances relevance and diversity while avoiding redundancy and contradictions.

\paragraph{Balancing Relevance and Diversity}

To adjust the trade-off between relevance and diversity, we introduce a hyperparameter \( \beta \), governing the selection process:
\begin{align*}
\log \det(L_{Y_g}) = \beta \cdot \sum_{i \in Y_g} \log(q_i^2) + & \\
(1 - \beta) \cdot \log \det(K_{\text{weighted}, Y_g})
\end{align*}
where \( \beta \in [0, 1] \) controls the emphasis on relevance versus diversity. Higher \( \beta \) values favor relevance, while lower values prioritize diversity and conflict minimization, ensuring that selected contexts are both coherent and non-redundant.

By modeling conflict relations and combining them with relevance and diversity, SMART delivers a robust, conflict-free context selection process that significantly improves the quality and coherence of generated answers in QA systems.

\section{Experiments}

\subsection{Experimental Setup}

We evaluate SMART on three diverse knowledge-intensive tasks using few-shot in-context learning~\cite{radford2019language,NEURIPS2020_1457c0d6,ram2023context}. Following prior work~\cite{trivedi2022interleaving}, we limit the experiments to the first 500 samples from each dataset to manage computational costs. Since test sets are unavailable for datasets from the KILT benchmark~\cite{Petroni2020KILTAB} (e.g., HotpotQA, FEVER), we report results on the development sets. The tasks include open-domain question answering, multi-hop reasoning, and fact verification, evaluated on established datasets such as NaturalQuestions (NQ), TriviaQA (TQA), HotpotQA, FEVER, and FOOL ME TWICE (FM2).  
An overview of the datasets, evaluation metrics, and experimental settings is provided in Table~\ref{tab:setting}.

\begin{table}[ht]
\centering
\resizebox{\columnwidth}{!}{%
\begin{tabular}{lccc}
\toprule
\textbf{Settings} & \textbf{Task} & \textbf{\#ICL} & \textbf{Metrics} \\
\midrule
\textbf{NQ} & ODQA & 5-shot & EM \\
\textbf{TQA} & ODQA & 5-shot & EM \\
\textbf{HotpotQA} & Multihop QA & 5-shot & F$_1$ \\
\textbf{FEVER} & Fact Verification & 5-shot & EM \\
\textbf{FM2} & Fact Verification & 5-shot & EM \\
\bottomrule
\end{tabular}%
}
\caption{Dataset statistics and experimental settings for NQ, TQA, HotpotQA, FEVER, and FM2 tasks.}
\label{tab:setting}
\end{table}

\subsection{Tasks and Datasets}
\label{sec:dataset_details}

\paragraph{Open-Domain Question Answering (QA)}  
For open-domain QA, we use the \textbf{NaturalQuestions (NQ)}~\cite{kwiatkowski2019natural} and \textbf{TriviaQA (TQA)}~\cite{joshi-etal-2017-triviaqa} datasets. NQ provides a set of questions \(q\) paired with short answers \(o\), and we use the processed version from~\citet{lee-etal-2019-latent}, focusing on instances where answers are five tokens or fewer. TriviaQA (TQA) consists of questions \(q\) with corresponding answers \(o\), extracted from supporting Wikipedia documents \(P\). We evaluate model performance on both datasets using the Exact Match (EM) metric, following the evaluation framework in~\citet{NEURIPS2020_6b493230}.

\paragraph{Multi-Hop Question Answering}  
For multi-hop QA, we adopt the \textbf{HotpotQA} dataset~\cite{yang2018hotpotqadatasetdiverseexplainable}, which requires reasoning over multiple passages \(P\) to answer each question \(q\). HotpotQA includes 113K question-answer pairs, many of which demand abstractive reasoning rather than direct extraction from supporting documents \(P\). As the answers often require deeper reasoning and synthesis, we use the unigram F$_1$ score for evaluation, consistent with~\citet{yang2018hotpotqadatasetdiverseexplainable}.

\paragraph{Fact Verification}  
To evaluate the system's ability to verify factual claims, we use the \textbf{FEVER} dataset~\cite{thorne2018feverlargescaledatasetfact}, a core part of the KILT benchmark~\cite{petroni2021kiltbenchmarkknowledgeintensive}. Each sample consists of a claim \(q\) that either supports or refutes information sourced from Wikipedia. Claims are labeled as "SUPPORTS" if they align with the facts and "REFUTES" if they contradict the facts. Model performance is assessed using accuracy, in line with previous work~\cite{thorne2018feverlargescaledatasetfact}. Additionally, we test on the \textbf{FOOL ME TWICE (FM2)} dataset~\cite{eisenschlos2021fooltwiceentailmentwikipedia}, which presents challenging claims designed to deceive players into selecting false facts. This dataset is particularly useful for evaluating the model's robustness in identifying subtle factual inconsistencies.

\subsection{Baseline Methods}
To evaluate the performance of SMART, we compare it against several well-established unsupervised baseline methods, which include relevance-focused and diversity-aware approaches. \textbf{BGE}~\cite{bge_embedding} selects the most relevant contexts based on the query-context relevance scores using the \textbf{bge-large-en-v1.5} model, focusing purely on relevance without considering diversity or conflicts. \textbf{Affinity Propagation}~\cite{wang2008adaptive} selects exemplar contexts from clusters to ensure diversity, while \textbf{Agglomerative Clustering}~\cite{7942580} hierarchically groups similar contexts and selects representatives from different clusters. \textbf{Greedy Selection} simply iteratively chooses the most dissimilar contexts to maximize diversity, without considering relevance. \textbf{Maximal Marginal Relevance (MMR)}~\cite{carbonell1998use} balances relevance and diversity by selecting contexts that are both relevant and dissimilar to previously chosen ones. \textbf{Non-negative Matrix Factorization (NMF)}~\cite{lee2000algorithms} groups contexts into topics and selects representatives from each to cover diverse aspects. \textbf{Spectral Clustering}~\cite{article} organizes contexts based on similarity and selects representatives to maintain diversity. \textbf{TextRank}~\cite{mihalcea2004textrank} is a graph-based method that ranks contexts by their centrality within a similarity graph, focusing on relevance, while \textbf{LexRank}~\cite{Erkan2004LexRankGL} also ranks contexts based on centrality but emphasizes representativeness of the dataset. For further details on each method, refer to Appendix~\ref{sec:baselines}.

\subsection{Implementation Details}
We evaluate the Llama3-8b-instruct model using the December 2018 Wikipedia dump as the retrieval corpus in a 5-shot in-context learning setup. Initially, we use Contriever~\cite{Izacard2021UnsupervisedDI} to retrieve the top 50 documents per query. We then apply BGE for pre-ranking, selecting the top 30 sentences with the highest BGE query-context scores. Finally, the SMART method and other context selection baselines are used to select the final set of sentences from this pre-ranked pool for input. Table~\ref{tab:optimal_params} shows the optimal \(\beta\) and \(\gamma\) values for each dataset used in our experiments. Additional details on decoding strategies, prompts, few-shot examples, context segmentation, and embedding models are provided in Appendix~\ref{sec:implement-details}.

\begin{table}[ht]
\centering
\resizebox{0.48\textwidth}{!}{%
\begin{tabular}{l|c|c|c|c|c}
\toprule
\textbf{Parameter} & \textbf{NQ} & \textbf{TQA} & \textbf{HotpotQA} & \textbf{FEVER} & \textbf{FM2} \\
\midrule
\textbf{Beta ($\beta$)}  & 0.8 & 0.9 & 0.9 & 0.9 & 0.8 \\
\textbf{Gamma ($\gamma$)} & 0.8 & 0.3 & 0.2 & 0.2 & 0.7 \\
\bottomrule
\end{tabular}%
}
\caption{Optimal hyperparameters (\(\beta\) and \(\gamma\)) used for each dataset.}
\label{tab:optimal_params}
\end{table}

\subsection{Experimental Results}
Table~\ref{tab:context-selection-contriever} presents the performance of SMART alongside baseline methods across multiple datasets, using Exact Match (EM) for NQ, TQA, FEVER, FM2, and F1 for HotpotQA.

\begin{table*}[ht]
\vspace{-1mm}
\centering
\resizebox{1\textwidth}{!}{
    \begin{tabular}{l|ccccc}
        \toprule
        \textbf{Method} & \textbf{\textsc{NQ} (EM)} & \textbf{\textsc{TQA} (EM)} & \textbf{\textsc{HotpotQA} (F1)} & \textbf{\textsc{FEVER} (EM)} & \textbf{\textsc{FM2} (EM)} \\
        \midrule
        \textsc{BGE}      & 25.1 & 48.9 & 38.8 & 91.1 & 72.0 \\
        \textsc{Contriever}      & 21.6 & 36.6 & 32.4 & 83.8 & 71.1 \\
        \textsc{AgglomerativeClustering}      & 21.0 & 40.8 & 33.9 & 86.7 & 67.3 \\
        \textsc{Greedy}                   & 21.2 & 33.5 & 29.5 & 78.4 & 65.5 \\
        \textsc{MMR}                      & 20.4 & 36.6 & 27.5 & 78.2 & 66.5 \\
        \textsc{NMF}                      & 21.6 & 40.6 & 32.5 & 81.8 & 72.1 \\
        \textsc{SpectralClustering}       & 22.4 & 41.8 & 36.4 & 89.7 & 72.3 \\
        \textsc{TextRank}                 & 23.4 & 47.3 & 33.4 & 86.3 & 68.7 \\
        \textsc{LexRank}                 & 23.2 & 47.7 & 33.4 & 86.5 & 69.1 \\
        \textsc{Random}                   & 21.6 & 38.0 & 30.4 & 84.2 & 67.3 \\
        \midrule
        \textsc{\textbf{SMART(Ours)}}                  & \textbf{25.9} & \textbf{49.3} & \textbf{40.8} & \textbf{93.5} & \textbf{75.6} \\
        \bottomrule
    \end{tabular}
}
\caption{Performance comparison of various context selection methods across datasets, using the primary evaluation metric mentioned in Table~\ref{tab:setting}.}
\vspace{-1mm}
\label{tab:context-selection-contriever}
\end{table*}

SMART consistently outperforms baseline methods across all datasets. It excels particularly in \textsc{FEVER} and \textsc{FM2}, where conflict resolution is critical, achieving EM scores of 93.5 and 75.6, respectively. SMART also demonstrates superior performance in \textsc{HotpotQA} (F1: 40.8), showcasing its effectiveness in multi-hop reasoning tasks. Across other datasets such as \textsc{NQ} and \textsc{TQA}, SMART maintains its advantage, highlighting its robustness and scalability in context selection for question answering.

\section{Analyses and Discussions}
\subsection{Ablation Experiments}

To evaluate the contributions of key components in the SMART framework, we conducted a series of ablation experiments by systematically modifying specific elements of the model. These experiments assess the impact of relevance and conflict modeling on overall performance across multiple datasets. The ablations focus on four key variations: (1) removing conflict resolution while retaining query-context relevance, (2) removing both relevance and conflict modeling, (3) using only BGE relevance scores, and (4) using the raw top 5 contexts retrieved by Contriever~\cite{Izacard2021UnsupervisedDI}.

\begin{table*}[ht]
\vspace{-1mm}
\centering
\resizebox{1\textwidth}{!}{
    \begin{tabular}{l|ccccc}
        \toprule
        \textbf{Method} & \textbf{\textsc{NQ} (EM)} & \textbf{\textsc{TQA} (EM)} & \textbf{\textsc{HotpotQA} (F1)} & \textbf{\textsc{FEVER} (Acc)} & \textbf{\textsc{FM2} (EM)} \\
        \midrule
        \textsc{\textbf{SMART (Full)}}                  & \textbf{25.9} & 49.3 & \textbf{40.8} & \textbf{93.5} & \textbf{75.6} \\
        \textsc{\textbf{SMART w/o Conflict}}            & 25.1 & \textbf{49.9} & 40.5 & 92.3 & 74.3 \\
        \textsc{\textbf{SMART w/o Relevance + Conflict}} & 23.2 & 37.8 & 31.3 & 78.4 & 66.3 \\
        \textsc{\textbf{SMART w/o Diversity + Conflict}}            & 25.1 & 48.9 & 38.8 & 91.1 & 72.0 \\
        \textsc{\textbf{SMART w/o Diversity + Conflict + Relevance}}               & 21.6 & 36.6 & 32.4 & 83.8 & 71.1 \\
        \bottomrule
    \end{tabular}
}
\caption{Ablation study comparing the performance of SMART with different settings for relevance and conflict modeling across datasets.}
\vspace{-1mm}
\label{tab:context-selection-ablation}
\end{table*}

\paragraph{SMART (Full)}  
The full SMART model combines both relevance and conflict modeling, constructing the kernel matrix \(L\) as:
$$ L = \text{Diag}(r) \cdot K_{\text{weighted}} \cdot \text{Diag}(r) $$
This approach effectively balances relevance, diversity, and conflict resolution, resulting in superior performance across all datasets. For instance, in \textsc{FEVER}, the model achieves an Accuracy of 93.5, outperforming other settings by a clear margin. Similarly, in \textsc{FM2}, SMART reaches an EM score of 75.6, highlighting the significant impact of conflict management on fact-based QA tasks. Across datasets like \textsc{HotpotQA}, where multi-hop reasoning is critical, the model also excels with an F1 score of 40.8, showcasing its ability to handle complex queries through diverse and coherent context selection.

\paragraph{SMART w/o Conflict}  
In this ablation, we removed conflict modeling while retaining query-context relevance and textual similarity. The kernel matrix \(L\) is simplified as:
$$ L = \text{Diag}(r) \cdot K_{\text{sim}} \cdot \text{Diag}(r) $$
While this variant continues to perform well, it struggles with selecting conflicting contexts, which impacts coherence and factual consistency. This is particularly evident in fact-based tasks like \textsc{FEVER} and \textsc{FM2}. For example, in \textsc{FEVER}, the Accuracy decreases by 1.2 points (93.5 to 92.3), and in \textsc{FM2}, the EM score drops by 1.3 points (75.6 to 74.3). These results underscore the critical role of conflict resolution in maintaining reliable, non-contradictory context selection.

\paragraph{SMART w/o Relevance + Conflict}  
In this ablation, both query-context relevance and conflict modeling are removed, leaving SMART to rely solely on DPP to maintain diversity using the similarity matrix:
$$ L = K_{\text{sim}} $$
While DPP promotes diverse context selection, this approach lacks mechanisms for ensuring the selected contexts are aligned with the query or free from contradictions. As a result, performance suffers significantly across all datasets, particularly in tasks requiring multi-hop reasoning and factual accuracy, such as \textsc{HotpotQA} and \textsc{FM2}. For instance, in \textsc{FM2}, the EM score drops by 9.3 points (from 75.6 to 66.3), underscoring the critical role of both relevance and conflict modeling for effective context selection.

\paragraph{SMART w/o Diversity + Conflict}  
In this configuration, context selection is based solely on the BGE query-context relevance score, without incorporating similarity or conflict modeling. While this approach performs reasonably well, it lacks the ability to manage conflicting contexts or ensure diversity. As a result, its performance is lower compared to the full SMART model. For instance, in \textsc{FM2}, the EM score is 72.0, falling short of the 75.6 achieved by the full model. Similarly, in \textsc{FEVER}, the Accuracy drops to 91.1, underscoring the importance of conflict resolution in tasks where factual consistency is critical. Although BGE Relevance Only performs better than more basic baselines, it remains limited by its inability to address conflicts and promote diversity.

\paragraph{SMART w/o Diversity + Conflict + Relevance}  
This baseline directly selects the top 5 sentences from documents retrieved by Contriever without any additional modeling. As expected, this approach yields the weakest performance, with notable declines across all datasets. The model is unable to handle conflicting or redundant contexts, particularly in fact-based tasks like \textsc{FEVER}, where Accuracy drops significantly to 83.8.

\paragraph{Key Findings}  
The results from the ablation experiments (Table~\ref{tab:context-selection-ablation}) clearly highlight the crucial role of both relevance and conflict modeling in SMART's performance. Removing conflict resolution leads to a significant decline in factual consistency, especially in datasets like \textsc{FM2} and \textsc{FEVER}, where maintaining coherence and accuracy is essential. The most pronounced performance drops are observed when both relevance and conflict modeling are excluded, underscoring the importance of integrating these components for high-quality, diverse, and coherent context selection. These findings reaffirm that managing relevance and conflict is essential for achieving optimal performance in fact-based QA tasks.

\subsection{Effects of Key Hyperparameters}

While SMART requires no additional training, tuning two key hyperparameters—$\beta$ and $\gamma$—is essential for optimizing performance. $\beta$ balances relevance and diversity, while $\gamma$ manages conflict resolution across contexts.

\paragraph{Relevance-Diversity Trade-off ($\beta$)}

$\beta$ controls the balance between relevance and diversity in context selection. As shown in Figure~\ref{fig:beta_gamma_show}, the best results are typically obtained with $\beta$ values between 0.7 and 0.8, depending on the dataset. For example, in FM2, $\beta = 0.8$ strikes an ideal balance, avoiding redundancy while maintaining relevance. In NQ, $\beta = 0.7$ provides sufficient diversity to cover multiple facets of the query. Setting $\beta$ too high leads to diminished performance due to redundant context selection.

\paragraph{Conflict Resolution ($\gamma$)}

$\gamma$ determines how strongly conflicting contexts are penalized. Its effects are especially pronounced in fact-based tasks like FEVER and FM2, where consistency is critical. For FEVER, $\gamma = 0.2$ provides the best balance, while FM2 benefits from a stricter penalty with $\gamma = 0.7$. In NQ, a higher $\gamma = 0.8$ helps eliminate contradictory information. However, setting $\gamma$ too high risks excluding useful, slightly conflicting contexts, reducing diversity and overall performance.

\begin{figure*}[htbp]
  \centering
  \includegraphics[width=\textwidth]{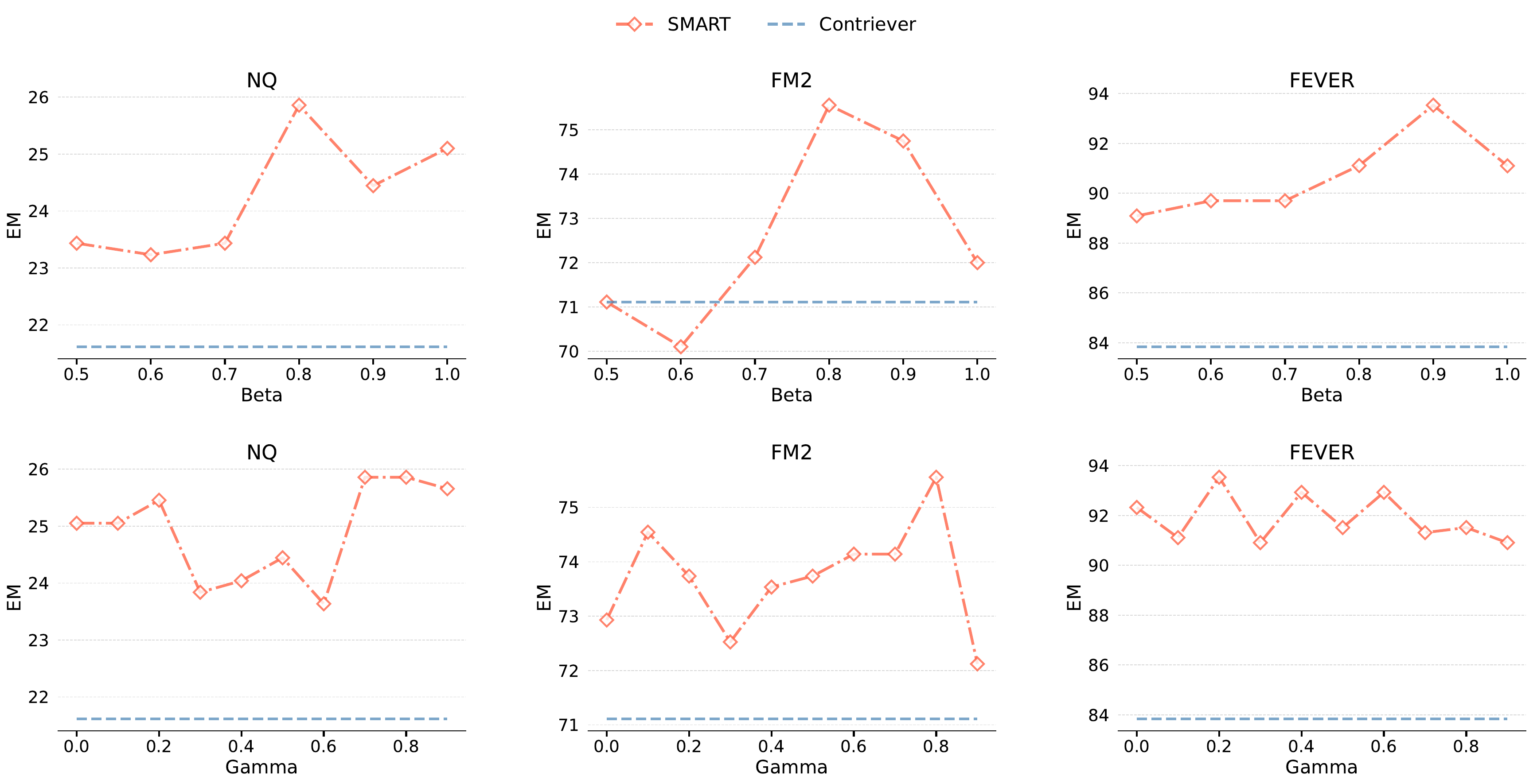}
  \caption{The different performance of SMART across datasets, with each hyperparameter (\(\beta\) or \(\gamma\)) tuned independently while keeping the other fixed at its optimal value (as reported in Table~\ref{tab:optimal_params}). The dashed lines represent the results of directly using top 5 sentences from documents retrieved by contriever.}
  \label{fig:beta_gamma_show}
\end{figure*}

\section{Related Work}

\subsection{Unsupervised Context Selection Methods}

Unsupervised context selection is crucial in low-data scenarios. Techniques like Affinity Propagation~\cite{wang2008adaptive} and TextRank~\cite{mihalcea2004textrank} focus on either relevance or diversity but often fail to balance both. Greedy Selection maximizes diversity by selecting the most dissimilar contexts, ignoring relevance. In contrast, Maximal Marginal Relevance (MMR)~\cite{carbonell1998use} balances relevance and diversity, choosing contexts that are both relevant and distinct. Non-negative Matrix Factorization (NMF)~\cite{lee2000algorithms} promotes topic diversity but does not directly prioritize relevance in its selections.

Recent unsupervised approaches such as REALM~\cite{guu2020realm} and REPLUG~\cite{shi2023replug} have reduced reliance on labeled data by involving some level of training, though they still require pre-training or fine-tuning. In contrast, SMART offers a fully train-free solution, integrating Determinantal Point Processes (DPPs) to jointly model relevance, diversity, and conflict. By addressing both query-context relationships and potential conflicts, SMART ensures a more balanced and effective selection of contexts, providing significant improvements in question-answering tasks without the need for extensive pre-training.

\subsection{Retrieval-Augmented Generation and Context Optimization}

Retrieval-Augmented Generation (RAG) improves language models by incorporating external information, but it often faces challenges such as noise, irrelevant details, and redundancy in the retrieved texts, which can degrade performance~\cite{rag,realm,searchain,wang2023self}. Most existing RAG context selection methods are supervised, relying heavily on labeled data, which limits their scalability in low-resource environments where such data is scarce~\cite{dong2024don}.

While unsupervised RAG approaches reduce the need for labeled data, they typically still require pre-training or fine-tuning and often struggle to balance relevance and diversity in context selection. SMART overcomes these limitations by introducing a fully unsupervised, train-free framework that leverages Determinantal Point Processes (DPPs). This ensures a balanced selection of contexts that minimizes redundancy, resolves conflicts, and optimizes relevance and diversity—filling a critical gap in existing RAG methodologies.

In augmented generation tasks, where additional contexts support the generation process, balancing relevance, diversity, and coherence is essential~\cite{lewis2020retrieval,guu2020realm}. Although retrieving longer passages may offer more information, they often introduce irrelevant or conflicting details that harm output quality~\cite{shi2023large}. SMART addresses this by selecting contexts that are both relevant and diverse while remaining free from contradictions, enhancing both factual accuracy and coherence in generated text.

Existing retrieval optimization techniques often focus on reranking based on relevance or utility~\cite{nogueira2020passage}, but neglect the need for diversity and conflict resolution. SMART improves upon these approaches by integrating conflict resolution, ensuring retrieved content is relevant, non-redundant, and consistent. This conflict-aware approach strengthens the performance of generation tasks, delivering more robust, coherent, and factually accurate outputs across diverse applications.

\section{Conclusion}

In this paper, we presented SMART, a novel unsupervised framework for enhancing context selection in retrieval-augmented generation (RAG). By incorporating conflict modeling into Determinantal Point Processes (DPPs), SMART addresses a key challenge in context selection: balancing relevance, diversity, and conflict resolution. Unlike traditional methods, SMART explicitly models and mitigates conflicts between retrieved contexts, ensuring that selected contexts are both coherent and factually consistent. This conflict-aware approach operates without the need for pre-training or fine-tuning, integrating query-context relevance and textual similarity to deliver high-quality outputs.

Extensive experiments across various datasets demonstrated that SMART outperforms existing methods, particularly in tasks such as fact verification and multi-hop reasoning. Its training-free nature makes it highly scalable for real-world applications, including virtual assistants, legal systems, and educational tools. By offering high-quality, conflict-free context selection, SMART has the potential to drive significant advancements in retrieval-augmented question answering and generation systems. 
\section*{Limitations}
A notable limitation of SMART is the time cost associated with computing conflict relations using the Natural Language Inference (NLI) model. While the DPP-based selection is efficient, the process of identifying contradictions between contexts through NLI introduces additional computational overhead, particularly when dealing with large-scale datasets. Moreover, the current sentence-level segmentation approach may not fully capture inter-sentence dependencies or broader document coherence, potentially missing key contextual relationships. Future work could focus on optimizing the NLI conflict detection process and exploring more sophisticated segmentation strategies to enhance both efficiency and context selection quality.
\section*{Acknowledgments}
This work was supported by National Key R\&D Program of China (2021YFF0901502), National Science Foundation of China (No. 62161160339), State Key Laboratory of Media Convergence Production Technology and Systems, and Key Laboratory of Science, Technology and Standard in Press Industry (Key Laboratory of Intelligent Press Media Technology). We appreciate the anonymous reviewers for their helpful comments. Xiaojun Wan is the corresponding author.

\bibliography{custom}

\clearpage
\appendix
\clearpage
\section{Detailed Baseline Methods}
\label{sec:baselines}
In this section, we provide a comprehensive explanation of the baseline methods used to evaluate SMART's performance in retrieval-augmented generation (RAG) for question answering (QA). Since SMART operates in a fully unsupervised and train-free manner, it offers a fair and meaningful comparison with other unsupervised baselines.

Below are the unsupervised methods used in our evaluation, including both traditional relevance-focused approaches and diversity-aware clustering techniques.

\paragraph{BGE} This baseline leverages the \texttt{bge-large-en-v1.5} model, which is a state-of-the-art query-context (QC) relevance calculator. The top contexts are selected based solely on their relevance scores. While this method focuses purely on maximizing relevance, it does not incorporate mechanisms for promoting diversity or addressing potential conflicts between the contexts. Due to the strength of the BGE model in capturing semantic relevance, this baseline represents a strong competitor, particularly in scenarios where relevance is the primary concern.

\paragraph{Affinity Propagation}
Affinity Propagation is a clustering-based algorithm that identifies exemplars, which are contexts that represent the most typical members of each cluster. It does this by exchanging messages between data points until a high-quality set of exemplars emerges. Once the clusters are formed, a representative context is selected from each cluster, ensuring diversity. However, this method does not explicitly prioritize relevance, leading to possible trade-offs in the final selection process.

\paragraph{Agglomerative Clustering}
Agglomerative Clustering is a hierarchical method that begins by treating each context as an individual cluster. The algorithm then iteratively merges pairs of clusters based on their similarity, creating a hierarchical tree (dendrogram). Once the clustering process is complete, contexts are selected from different clusters to promote diversity. This method ensures that the final selection covers a wide range of information, although it does not explicitly prioritize relevance to the query.

\paragraph{Greedy Selection}
Greedy Selection is an iterative approach that selects contexts one by one, focusing exclusively on maximizing diversity. Each new context is chosen to be as dissimilar as possible from the already selected contexts, ensuring that the final set covers a broad range of information. However, this method sacrifices relevance to the query, as it does not incorporate any mechanisms for optimizing relevance scores.

\paragraph{Maximal Marginal Relevance (MMR)}
MMR is a well-known method for balancing relevance and diversity. It iteratively selects contexts that are highly relevant to the query while penalizing those that are too similar to the already selected contexts. This ensures that the selected contexts are both relevant and diverse, making MMR a more balanced approach than methods focused solely on relevance or diversity.

\paragraph{Non-negative Matrix Factorization (NMF)}
NMF is a factorization technique that decomposes the query-context similarity matrix into a set of latent components, where each component represents a topic. Contexts are then grouped into these topics, and a representative context is selected from each topic to ensure topic-wise coverage. While NMF promotes diversity by capturing different aspects of the data, it does not handle conflicts or directly optimize relevance for each selected context.

\paragraph{Spectral Clustering}
Spectral Clustering uses the eigenvalues of the similarity matrix to perform dimensionality reduction before clustering. By capturing the underlying structure of the context space, Spectral Clustering groups similar contexts and selects representative contexts from each cluster. This method promotes diversity, ensuring that a range of perspectives is covered, though it does not explicitly model relevance or conflict.

\paragraph{TextRank}
TextRank is a graph-based algorithm used for text ranking, similar to Google's PageRank. In the context of RAG tasks, TextRank represents contexts as nodes in a graph, with edges weighted by the cosine similarity between contexts. The algorithm assigns higher ranks to more central contexts (those connected to other relevant contexts), making it an effective method for identifying key contexts based on their relevance within the set. However, it does not address conflict or redundancy directly.

\paragraph{LexRank}
LexRank is another graph-based method similar to TextRank but focuses more on centrality within the context graph. LexRank ranks contexts based on their importance, determined by their connection to other contexts in the graph. It ensures that the selected contexts are representative of the whole dataset, promoting diversity. However, like TextRank, LexRank does not handle conflicts or redundancy.

\paragraph{Random Selection}
This baseline randomly selects contexts from the pool, without any consideration for relevance, diversity, or conflict. Although simplistic, it serves as a lower-bound comparison, illustrating the importance of more sophisticated context selection methods in RAG tasks.

\section{Implementation Details}
\label{sec:implement-details}
To ensure reproducibility, we provide a detailed overview of the key implementation settings and parameters used in our experiments.

\subsection{Data Retrieval and Preprocessing}
\paragraph{Retrieval and Context Segmentation}
We use the Contriever model to retrieve the top 50 documents per query. After retrieval, documents are split into individual sentences using the Natural Language Toolkit (NLTK)~\cite{bird-loper-2004-nltk}. This sentence-level segmentation allows for more precise analysis. The sentences are then re-ranked using the \texttt{bge-large-en-v1.5} embedding model~\cite{bge_embedding}, ensuring more relevant and focused context selection. After pre-ranking, SMART and other baseline methods are used to select the final sentences from the pre-ranked pool, optimizing the balance between relevance, diversity, and conflict resolution.  

\paragraph{BGE preranking}
In the SMART framework, relations between contexts and queries are established at the sentence level. However, generating relation matrices for all sentences is inefficient, as many sentences are of low quality or noisy. Moreover, processing all relations—particularly the context-context conflict relation, with its complexity of \(O(n^2)\)—becomes computationally prohibitive as \(n\) increases. To address this, we employ a pre-ranking step using the BGE ranker, selecting the top 30 sentences. Our method is then applied to choose the top 5 sentences from this set. All baseline methods also utilize this same pre-ranking step to ensure a fair comparison. For initial document retrieval, we use the Contriever model, which retrieves the top 50 documents. These documents are then split into sentences, and the BGE ranker is applied at the sentence level.

\subsection{Models}
\paragraph{QA models}
For all QA tasks, we use the Llama3-8b-instruct model~\cite{llama3modelcard}. Greedy decoding is applied with a temperature setting of \( \text{temp} = 0 \), ensuring deterministic outputs by always selecting the most probable token at each step.
\paragraph{Embedding Model}
For dense vector representations of contexts and queries, we utilize the \texttt{bge-large-en-v1.5} embedding model~\cite{bge_embedding}, which captures semantic similarities across textual inputs.
\paragraph{Natural Language Inference (NLI) Model}
To detect and resolve conflicts between contexts, we employ the DeBERTa-v3 model~\cite{sileo2023tasksource}. This model identifies contradictions and entailments between text pairs, crucial for maintaining coherence in selected contexts.

\subsection{Hyperparameters}
The number of selected sentences, \( \text{top}_k \), is fixed at 5. We perform a grid search over a subset of the development set to determine the optimal hyperparameters, with the results detailed in Table~\ref{tab:optimal_params}. The beta parameter \( \beta \), which controls the trade-off between relevance and diversity, is adjusted across values \( \{0.5, 0.6, 0.7, 0.8, 0.9, 1.0\} \). For \(\beta = 1.0\), the model relies almost entirely on relevance, and we provide the BGE-ranking-only results for comparison. The gamma parameter \( \gamma \), which influences the handling of conflict resolution, is tuned within the range \( \{0.1, 0.2, \dots, 0.9\} \). 

\subsection{Hardware and Computational Efficiency}
\paragraph{Hardware Setup}
All experiments were conducted on an NVIDIA A40 GPU.

\paragraph{Computational Efficiency}
The construction of the kernel matrix \( L \) for Determinantal Point Processes (DPP) involves a time complexity of \( O(n^2) \), where \( n \) represents the number of contexts. Once the kernel matrix is constructed, the DPP sampling operates with a complexity of \( O(n^3) \), providing an efficient framework for context selection without requiring extensive model training. Given the absence of any training phases, SMART is computationally efficient and scalable, suitable for real-time or large-scale applications with constrained resources.

\paragraph{Prompts and Few-Shot Examples}
The prompts used for question answering (QA) and query decomposition are crucial to the performance of our method. The specific prompts, along with the few-shot examples employed, are detailed in Appendix~\ref{sec:prompt}. These examples guide the model in generating contextually appropriate and accurate responses.


\section{Algorithm}
\begin{algorithm}[!h]
    \caption{Data Preprocessing for SMART}
    \label{alg:preprocessing}
    \renewcommand{\algorithmicrequire}{\textbf{Input:}}
    \renewcommand{\algorithmicensure}{\textbf{Output:}}
    
    \begin{algorithmic}[1]
        \REQUIRE Contexts $C$, Queries $Q$
        \ENSURE Preprocessed contexts $\mathcal{C}$, Preprocessed queries $\mathcal{Q}$, \\
        \hspace*{\algorithmicindent} Context similarity matrix $\mathbf{S}_{\text{cc}}$, \\
        \hspace*{\algorithmicindent} Context conflict matrix $\mathbf{C}_{\text{cc}}$, \\
        \hspace*{\algorithmicindent} Query-Context relevance matrix $\mathbf{S}_{\text{qc}}$
        
        \STATE Initialize empty lists $\mathcal{C}$ and $\mathcal{Q}$
        \STATE Initialize empty matrices $\mathbf{S}_{\text{cc}}$, $\mathbf{C}_{\text{cc}}$, $\mathbf{S}_{\text{qc}}$
        
        \FOR{each $c \in C$}
            \STATE Segment context $c$ into individual sentences \(\{s_1, s_2, \dots, s_n\}\) using NLTK
            \STATE Add segmented context to $\mathcal{C}$
        \ENDFOR
        
        \FOR{each $c_i, c_j \in \mathcal{C}$}
            \STATE Compute cosine similarity between $c_i$ and $c_j$ using dense vectors from BGE
            \STATE Store similarity in $\mathbf{S}_{\text{cc}}[i,j]$
            \STATE Compute conflict between $c_i$ and $c_j$ using NLI models
            \STATE Store conflict probability in $\mathbf{C}_{\text{cc}}[i,j]$
        \ENDFOR
        
        \FOR{each $q \in Q$, $c \in \mathcal{C}$}
            \STATE Compute cosine similarity between $q$ and $c$ using BGE to measure the relevance of $c$ to $q$
            \STATE Store relevance in $\mathbf{S}_{\text{qc}}[q,c]$
        \ENDFOR
        
        \RETURN $\mathcal{C}$, $\mathcal{Q}$, $\mathbf{S}_{\text{cc}}$, $\mathbf{C}_{\text{cc}}$, $\mathbf{S}_{\text{qc}}$
    \end{algorithmic}
\end{algorithm}

\begin{algorithm}[!h]
    \caption{DPP-based Context Selection with Symmetrized Conflict Matrix}
    \label{alg:DPP_Selection_Symmetric}
    \renewcommand{\algorithmicrequire}{\textbf{Input:}}
    \renewcommand{\algorithmicensure}{\textbf{Output:}}
    
    \begin{algorithmic}[1]
        \REQUIRE Preprocessed contexts $\mathcal{C}$, Context relevance vector $q$, \\
        \hspace*{\algorithmicindent} Context similarity matrix $K_{\text{sim}}$, \\
        \hspace*{\algorithmicindent} Conflict matrix $C$, \\
        \hspace*{\algorithmicindent} Hyperparameters $\gamma$ and $\theta$, Number of contexts to select $k$
        \ENSURE Selected contexts $\mathcal{C}_{\text{selected}}$
        
        \STATE Compute the symmetrized conflict matrix \( C' \):
        \[
        C'_{ij} = \frac{1}{2} \left( C_{ij} + C_{ji} \right)
        \]
        
        \STATE Initialize the weighted similarity matrix:
        \[
        K_{\text{weighted}} = K_{\text{sim}} \circ \exp(-\gamma (1 - C'))
        \]
        
        \STATE Construct the kernel matrix \( L \) for DPP:
        \[
        L = \text{Diag}(q) \cdot K_{\text{weighted}} \cdot \text{Diag}(q)
        \]
        
        \STATE Initialize an empty set $\mathcal{C}_{\text{selected}}$
        
        \WHILE{$|\mathcal{C}_{\text{selected}}| < k$}
            \STATE For each context $c \in \mathcal{C} \setminus \mathcal{C}_{\text{selected}}$, compute the marginal gain:
            \[
            \Delta_c = \log \det(L_{\mathcal{C}_{\text{selected}} \cup \{c\}}) - \log \det(L_{\mathcal{C}_{\text{selected}}})
            \]
            \STATE Select the context $c_{\text{best}}$ that maximizes $\Delta_c$:
            \[
            c_{\text{best}} = \mathop{\arg\max}_{c \in \mathcal{C} \setminus \mathcal{C}_{\text{selected}}} \Delta_c
            \]
            \STATE Add $c_{\text{best}}$ to $\mathcal{C}_{\text{selected}}$
        \ENDWHILE
        
        \STATE Adjust the selected contexts by tuning the balance between relevance and diversity:
        \begin{equation*}
        \begin{aligned}
        \log \det(L_{\mathcal{C}_{\text{selected}}}) &= \beta \cdot \sum_{i \in \mathcal{C}_{\text{selected}}} \log(q_i^2) \\
        &\quad + (1 - \beta) \cdot \log \det(S_{\mathcal{C}_{\text{selected}}})
        \end{aligned}
        \end{equation*}
        
        \RETURN $\mathcal{C}_{\text{selected}}$
    \end{algorithmic}
\end{algorithm}

\clearpage
\section{Theoretical Proofs}
\label{sec:proofs}
\subsection{Proof of Proposition 3.1(DPP kernel PSD)}
\begin{theorem}
The kernel matrix $L$ of a Determinantal Point Process (DPP) must be positive semi-definite (PSD).
\end{theorem}

\begin{proof}
A Determinantal Point Process (DPP) is a probabilistic model used to define distributions over subsets of a ground set $\mathcal{Y} = \{1, 2, \dots, N\}$. The DPP is parameterized by a kernel matrix $L \in \mathbb{R}^{N \times N}$, where each entry $L_{ij}$ captures the similarity between items $i$ and $j$. The DPP defines a probability distribution over all subsets $Y \subseteq \mathcal{Y}$ as:
\[
\mathbb{P}(Y) \propto \det(L_Y),
\]
where $L_Y$ is the submatrix of $L$ indexed by the elements of $Y$. \\ \\
For the DPP model to be valid, the kernel matrix $L$ must satisfy two key conditions: (1) The matrix $L$ must be symmetric, i.e., $L_{ij} = L_{ji}$ for all $i, j \in \mathcal{Y}$. (2) The matrix $L$ must be positive semi-definite (PSD), which requires that for any vector $\mathbf{x} \in \mathbb{R}^N$,
\[
\mathbf{x}^\top L \mathbf{x} \geq 0.
\] 
The necessity for $L$ to be PSD arises from the properties of the determinant in the DPP model. Specifically, the probability assigned to any subset $Y$ is proportional to the determinant of the submatrix $L_Y$. For this probability to be non-negative, $\det(L_Y)$ must be non-negative for all subsets $Y$. This condition is satisfied if $L$ is PSD, as the determinant of any principal submatrix of a PSD matrix is non-negative. Additionally, the probability $\mathbb{P}(Y)$ of selecting a subset $Y$ must be non-negative, implying that $\det(L_Y) \geq 0$. If $L$ were not PSD, some submatrices $L_Y$ could have negative determinants, resulting in invalid (negative) probabilities. \\ \\
To ensure that $\mathbb{P}(Y) \geq 0$ for all subsets $Y$, the matrix $L$ must be such that for any vector $\mathbf{x} \in \mathbb{R}^N$, the quadratic form $\mathbf{x}^\top L \mathbf{x} \geq 0$ holds. This condition implies that all eigenvalues of $L$ must be non-negative, which is a defining characteristic of a PSD matrix. \\ \\
Therefore, the kernel matrix $L$ of a DPP must be positive semi-definite because the determinants of all principal submatrices $L_Y$ must be non-negative to ensure valid (non-negative) probabilities, and the quadratic form $\mathbf{x}^\top L \mathbf{x} \geq 0$ must hold for all $\mathbf{x} \in \mathbb{R}^N$, necessitating that $L$ be PSD. Hence, the positive semi-definiteness of the kernel matrix $L$ is a fundamental requirement for the DPP model to define a valid probability distribution.
\end{proof}

\subsection{Proof of Proposition 3.2(CosSim PSD)}
\begin{theorem}
The cosine similarity matrix $A \in \mathbb{R}^{n \times n}$, where each element $A_{ij}$ represents the cosine similarity between vectors $v_i$ and $v_j$, is positive semi-definite (PSD).
\end{theorem}
\begin{proof}
The cosine similarity between two vectors $v_i$ and $v_j$ in $\mathbb{R}^d$ is defined as:
\[
A_{ij} = \cos(\theta_{ij}) = \frac{v_i^\top v_j}{\|v_i\|\|v_j\|},
\]
where $\theta_{ij}$ is the angle between the vectors $v_i$ and $v_j$, and $\|v_i\|$ denotes the Euclidean norm of vector $v_i$. The matrix $A$ is symmetric, with each entry $A_{ij}$ representing the cosine similarity between the $i$-th and $j$-th vectors.\\ \\
Let $V = [v_1, v_2, \dots, v_n] \in \mathbb{R}^{d \times n}$ be a matrix whose columns are the vectors $v_i$. The cosine similarity matrix $A$ can be expressed as:
\[
A = D^{-1} V^\top V D^{-1},
\]
where $D$ is a diagonal matrix with entries $D_{ii} = \|v_i\|$. Here, $V^\top V$ is the Gram matrix of the vectors $v_i$. \\ \\
To prove that $A$ is positive semi-definite, we must show that for any vector $x \in \mathbb{R}^n$, the inequality $x^\top A x \geq 0$ holds. \\ \\
Using the expression for $A$, we write:
\[
x^\top A x = x^\top D^{-1} V^\top V D^{-1} x.
\]
Letting $y = D^{-1} x$, we have:
\[
x^\top A x = y^\top V^\top V y = (Vy)^\top (Vy) = \|Vy\|^2.
\]
Since the squared Euclidean norm $\|Vy\|^2 \geq 0$ for all vectors $y$, it follows that:
\[
x^\top A x \geq 0.
\]
Thus, the cosine similarity matrix $A$ is positive semi-definite.
\end{proof}

\subsection{Proof of Proposition 3.3(C not PSD)}
\begin{theorem}
The conflict matrix $C$, constructed from contradiction probabilities derived from a Natural Language Inference (NLI) model, is not guaranteed to be positive semi-definite (PSD) because it can be non-symmetric and may have negative eigenvalues.
\end{theorem}

\begin{proof}
Let $C$ be an $n \times n$ matrix where each element $C_{ij}$ represents the probability of contradiction between two sentences $C_i$ and $C_j$ in an NLI framework:
\begin{align*}
C_{ij} &= P(\text{contradiction} \mid \text{premise} = C_i, \\
          &\quad \text{hypothesis} = C_j).
\end{align*} 
In this context, one sentence is treated as the premise and the other as the hypothesis. Consequently, the contradiction probability $C_{ij}$ may differ from $C_{ji}$, implying that the conflict matrix $C$ is not necessarily symmetric. \\ \\ 
For a matrix to be positive semi-definite, it must satisfy two conditions: it must be symmetric, meaning $C_{ij} = C_{ji}$ for all $i$ and $j$, and all of its eigenvalues must be non-negative. \\ \\ 
Given that $C$ is not symmetric in general, it does not satisfy the first condition necessary for positive semi-definiteness. Moreover, non-symmetric matrices can have complex or negative eigenvalues, further challenging their positive semi-definiteness. \\ \\
To demonstrate this, consider the following specific case where the conflict matrix is non-symmetric: \\ \\
Let $n = 3$, and consider the conflict matrix $C$ for three sentences $C_1$, $C_2$, and $C_3$:
\[
C = \begin{pmatrix}
0 & 0.8 & 0.7 \\
0.8 & 0 & 0.9 \\
0.7 & 0.8 & 0
\end{pmatrix}.
\]
This matrix is clearly non-symmetric, as $C_{ij} \neq C_{ji}$ for some pairs $(i, j)$. The eigenvalues of this matrix are approximately:
\[
\lambda_1 \approx 1.568, \quad \lambda_2 \approx -0.868, \quad \lambda_3 \approx -0.7.
\]
The presence of negative eigenvalues $\lambda_2$ and $\lambda_3$ indicates that $C$ is not positive semi-definite. \\ \\
Therefore, the non-symmetry and the existence of negative eigenvalues in the conflict matrix $C$ demonstrate that it cannot be guaranteed to be positive semi-definite. This outcome is expected, as positive semi-definiteness requires both symmetry and non-negative eigenvalues, conditions that are not inherently satisfied by the conflict matrix in this context.
\end{proof}

\subsection{Proof of Proposition 3.4(C' not PSD)}
\begin{theorem}
The conflict matrix \( C \), constructed from the average of contradiction probabilities derived from a Natural Language Inference (NLI) model, is not guaranteed to be positive semi-definite (PSD) even after symmetrization. Although symmetrizing \( C \) ensures it is symmetric, it may still have negative eigenvalues, thus failing to satisfy the criteria for positive semi-definiteness.
\end{theorem}

\begin{proof}
Let \( C \) be an \( n \times n \) matrix where each element \( C_{ij} \) represents the symmetrized probability of contradiction between two sentences \( C_i \) and \( C_j \) in an NLI framework. The symmetrized conflict matrix is constructed as follows:
\begin{align*}
C_{ij} = \frac{1}{2} \big( & P(C_i \rightarrow C_j) \\
& + P(C_j \rightarrow C_i) \big).
\end{align*}

This symmetrization ensures that \( C_{ij} = C_{ji} \) for all \( i \) and \( j \), making the conflict matrix symmetric. Symmetry is a necessary condition for positive semi-definiteness, but it is not sufficient on its own. For \( C \) to be PSD, all eigenvalues of \( C \) must also be non-negative.

However, even with symmetrization, the conflict matrix may have negative eigenvalues, indicating that it is not necessarily PSD. To demonstrate this, consider the following specific case where the conflict matrix is symmetrized but still fails to be PSD:

Let \( n = 3 \), and consider the conflict matrix \( C \) for three sentences \( C_1 \), \( C_2 \), and \( C_3 \). Suppose the initial (non-symmetric) contradiction probabilities are given by:
\[
\begin{aligned}
P(C_1 \rightarrow C_2) &= 0.8, &\quad P(C_2 \rightarrow C_1) &= 0.6, \\
P(C_1 \rightarrow C_3) &= 0.7, &\quad P(C_3 \rightarrow C_1) &= 0.5, \\
P(C_2 \rightarrow C_3) &= 0.9, &\quad P(C_3 \rightarrow C_2) &= 0.8.
\end{aligned}
\]

The symmetrized conflict matrix \( C \) is then:
\[
C = \begin{pmatrix}
0 & 0.7 & 0.6 \\
0.7 & 0 & 0.85 \\
0.6 & 0.85 & 0
\end{pmatrix}.
\]
This matrix is symmetric, satisfying the first condition for being PSD. However, to determine whether \( C \) is PSD, we must examine its eigenvalues.

The eigenvalues of this matrix are approximately:
\[
\lambda_1 \approx 1.438, \quad \lambda_2 \approx -0.864, \quad \lambda_3 \approx -0.575.
\]
The presence of negative eigenvalues \( \lambda_2 \) and \( \lambda_3 \) indicates that \( C \) is not positive semi-definite, despite being symmetric.

Therefore, even though the symmetrization process \( C_{ij} = \frac{1}{2} (C_{ij} + C_{ji}) \) guarantees that the conflict matrix is symmetric, it does not ensure positive semi-definiteness. The conflict matrix may still have negative eigenvalues, which means it does not satisfy the necessary condition for positive semi-definiteness. This outcome highlights the limitation of merely symmetrizing the conflict matrix in guaranteeing its positive semi-definiteness.
\end{proof}

\subsection{Proof of Proposition 3.5(exp PSD)}
\begin{theorem}
For any symmetric matrix \( C \in \mathbb{R}^{n \times n} \) with all elements \( C_{ij} \geq 0 \) and any non-negative scalar \( \gamma \geq 0 \), the matrix \( \exp(-\gamma C) \) is positive semi-definite (PSD).
\end{theorem}

\begin{proof}
Let \( C \) be an \( n \times n \) symmetric matrix where \( C_{ij} \geq 0 \) for all \( i, j \), and let \( \gamma \geq 0 \) be a non-negative scalar.

Since \( C \) is symmetric, \( C = C^\top \), meaning \( C \) can be diagonalized by an orthogonal matrix. Specifically, there exists an orthogonal matrix \( V \) and a diagonal matrix \( \Lambda \) such that:
\[
C = V \Lambda V^\top,
\]
where \( \Lambda \) is a diagonal matrix whose diagonal entries \( \lambda_1, \lambda_2, \dots, \lambda_n \) are the eigenvalues of \( C \).

The matrix exponential of \( -\gamma C \) is given by:
\[
\exp(-\gamma C) = V \exp(-\gamma \Lambda) V^\top,
\]
where \( \exp(-\gamma \Lambda) \) is a diagonal matrix with entries \( \exp(-\gamma \lambda_1), \exp(-\gamma \lambda_2), \dots, \exp(-\gamma \lambda_n) \) on the diagonal.

Since \( C \) is symmetric and all elements \( C_{ij} \geq 0 \), the eigenvalues \( \lambda_1, \lambda_2, \dots, \lambda_n \) of \( C \) are real. Because \( C_{ij} \geq 0 \), the eigenvalues are also non-negative (by the Perron-Frobenius theorem for non-negative symmetric matrices, which states that all eigenvalues are non-negative)\cite{Perron1907ZurTD}.

Now consider the matrix \( \exp(-\gamma \Lambda) \), where \( \exp(-\gamma \lambda_i) \) for each eigenvalue \( \lambda_i \) is given by:
\[
\exp(-\gamma \lambda_i) > 0,
\]
because the exponential function \( \exp(x) \) is positive for any real number \( x \). Therefore, all the diagonal entries of \( \exp(-\gamma \Lambda) \) are positive.

The matrix \( \exp(-\gamma C) = V \exp(-\gamma \Lambda) V^\top \) is a similarity transformation of the diagonal matrix \( \exp(-\gamma \Lambda) \). Since \( \exp(-\gamma \Lambda) \) is diagonal with positive entries, it is positive definite, and hence PSD.

Because similarity transformations preserve eigenvalues, the eigenvalues of \( \exp(-\gamma C) \) are the same as those of \( \exp(-\gamma \Lambda) \), which are positive. Therefore, \( \exp(-\gamma C) \) is positive definite, and thus positive semi-definite (PSD).

Since \( \exp(-\gamma C) \) is symmetric and all of its eigenvalues are non-negative, it is positive semi-definite. This completes the proof.
\end{proof}

\subsection{Proof of Proposition 3.6(Product PSD)}
\begin{theorem}
Given a positive semi-definite (PSD) similarity matrix $K_{\text{sim}} \in \mathbb{R}^{n \times n}$ and a conflict matrix $C \in \mathbb{R}^{n \times n}$, the weighted similarity matrix $K_{\text{weighted}} = K_{\text{sim}} \circ \exp(-\gamma C)$, where $\circ$ denotes the element-wise (Hadamard) product and $\exp(-\gamma C)$ applies an element-wise exponential function to $C$ with $\gamma \geq 0$, is guaranteed to be positive semi-definite.
\end{theorem}

\begin{proof}
Let $K_{\text{sim}} \in \mathbb{R}^{n \times n}$ be a PSD matrix. By definition, this means that for any vector $\mathbf{x} \in \mathbb{R}^n$, the quadratic form associated with $K_{\text{sim}}$ satisfies:
\[
\mathbf{x}^\top K_{\text{sim}} \mathbf{x} \geq 0.
\]
Let $C \in \mathbb{R}^{n \times n}$ be a conflict matrix, and let $\gamma \geq 0$ be a non-negative scalar.\\ \\
Define $\exp(-\gamma C)$ as the element-wise exponential of $-\gamma C$, meaning that each element $[\exp(-\gamma C)]_{ij}$ is given by:
\[
[\exp(-\gamma C)]_{ij} = \exp(-\gamma C_{ij}),
\]
where $\exp(-\gamma C_{ij}) > 0$ for all $i, j$. \\ \\
The Hadamard product of two matrices $A$ and $B$ (denoted $A \circ B$) is a matrix $D$ where each element $D_{ij} = A_{ij} B_{ij}$. \\ \\
A key result from Schur product theorem is that the Hadamard product of two PSD matrices is also PSD, provided that both matrices are PSD themselves\cite{Schur+1911+1+28}. \\ \\
Next, consider the matrix $\exp(-\gamma C)$. Since $C$ is symetric and each element of $C > 0$, from Proposition 3.6 we know $\exp(-\gamma C)$ is also PSD. \\ \\
Now, define the weighted similarity matrix as:
\[
K_{\text{weighted}} = K_{\text{sim}} \circ \exp(-\gamma C).
\]
$K_{\text{weighted}}$ is the Hadamard product of $K_{\text{sim}}$ and $\exp(-\gamma C)$. \\ \\
Since $K_{\text{sim}}$ is PSD by assumption and $\exp(-\gamma C)$ is positive definite (and therefore PSD), their Hadamard product $K_{\text{weighted}}$ is also PSD. \\ \\
Specifically, for any vector $\mathbf{x} \in \mathbb{R}^n$, the quadratic form for $K_{\text{weighted}}$ can be expressed as:
\[
\mathbf{x}^\top K_{\text{weighted}} \mathbf{x} = \sum_{i=1}^n \sum_{j=1}^n x_i x_j K_{\text{sim},ij} \exp(-\gamma C_{ij}).
\]
Since $K_{\text{sim},ij} \geq 0$ and $\exp(-\gamma C_{ij}) > 0$, it follows that:
\[
\mathbf{x}^\top K_{\text{weighted}} \mathbf{x} \geq 0 \quad \text{for all } \mathbf{x} \in \mathbb{R}^n.
\]
Therefore, $K_{\text{weighted}}$ is PSD.
\end{proof}

\begin{corollary}
The weighted similarity matrix $K_{\text{weighted}} = K_{\text{sim}} \circ \exp(-\gamma C)$ is guaranteed to be positive semi-definite as long as $K_{\text{sim}}$ is PSD and $\exp(-\gamma C)$ is a matrix with positive elements, which it is by construction. The Hadamard product of two PSD matrices is PSD, ensuring that $K_{\text{weighted}}$ retains the PSD property.
\end{corollary}

\subsection{Proof of Proposition 3.7 (C Reduce)}
\label{sec:c_reduce}
\begin{theorem}
Given a positive semi-definite (PSD) similarity matrix \( K_{\text{sim}} \in \mathbb{R}^{n \times n} \) and a conflict matrix \( C \in \mathbb{R}^{n \times n} \), the incorporation of conflict information into the Determinantal Point Process (DPP) via the weighted similarity matrix:
\[
K_{\text{weighted}} = K_{\text{sim}} \circ \exp(-\gamma (1 - C)),
\]
where \( \gamma \geq 0 \), \textbf{increases} the probability of selecting pairs of sentences with low conflict and \textbf{reduces} the probability of selecting pairs of sentences with high conflict.
\end{theorem}

\begin{proof}
Let \( Y = \{1, 2, \dots, n\} \) represent a set of sentences. In a Determinantal Point Process (DPP), the probability distribution over subsets \( Y_g \subseteq Y \) is determined by the determinant of the submatrix \( L_{Y_g} \), which corresponds to the elements in \( Y_g \). The probability of selecting a subset \( Y_g \) is given by
\[
P(Y_g) \propto \det(L_{Y_g}),
\]
where \( L_{Y_g} \) is the submatrix of \( L \) corresponding to the sentences in \( Y_g \).

The kernel matrix \( L \) is constructed as
\[
L = \text{Diag}(q) \cdot K_{\text{weighted}} \cdot \text{Diag}(q),
\]
where \( K_{\text{weighted}} = K_{\text{sim}} \circ \exp(-\gamma (1 - C)) \), and \( q \in \mathbb{R}^n \) is a vector of relevance scores, with \( \text{Diag}(q) \) representing a diagonal matrix whose entries are the elements of \( q \). The matrix \( K_{\text{sim}} \) encodes the pairwise similarity between sentences, and \( C \) represents the conflict matrix, where \( C_{ij} \in [0, 1] \) denotes the level of conflict between sentences \( i \) and \( j \).

The weighted similarity matrix \( K_{\text{weighted}} \) incorporates conflict information through an element-wise exponential transformation:
\[
K_{\text{weighted}, ij} = K_{\text{sim}, ij} \cdot \exp(-\gamma (1 - C_{ij})).
\]
This ensures that when \( C_{ij} \) (the conflict between sentences \( i \) and \( j \)) is large, \( \exp(-\gamma (1 - C_{ij})) \) is close to 1, preserving the similarity \( K_{\text{sim}, ij} \). In contrast, when \( C_{ij} \) is small, \( \exp(-\gamma (1 - C_{ij})) \) is close to 0, reducing the similarity and encouraging the selection of non-conflicting sentences.

The kernel matrix \( L \) is influenced by the conflict-adjusted similarity matrix \( K_{\text{weighted}} \). Specifically, the off-diagonal elements \( L_{ij} \) and \( L_{ji} \), which encode the relationship between sentences \( i \) and \( j \), are given by
\begin{align*}
L_{ij} &= q_i \cdot K_{\text{weighted}, ij} \cdot q_j \\
       &= q_i \cdot K_{\text{sim}, ij} \cdot \exp\left(-\gamma (1 - C_{ij})\right) \cdot q_j.
\end{align*}

This reduces \( L_{ij} \) when \( C_{ij} \) is small, corresponding to high conflict between sentences \( i \) and \( j \).

Consider a subset \( Y_g = \{i, j\} \) consisting of two sentences \( i \) and \( j \). The determinant of the corresponding \( 2 \times 2 \) submatrix \( L_{Y_g} \) is
\[
\det(L_{Y_g}) = \det\left( \begin{pmatrix} L_{ii} & L_{ij} \\ L_{ji} & L_{jj} \end{pmatrix} \right) = L_{ii} L_{jj} - L_{ij}^2,
\]
where \( L_{ii} = q_i^2 \cdot K_{\text{sim}, ii} \) and \( L_{jj} = q_j^2 \cdot K_{\text{sim}, jj} \) represent the relevance of sentences \( i \) and \( j \), and \( L_{ij} = q_i \cdot K_{\text{sim}, ij} \cdot \exp(-\gamma (1 - C_{ij})) \cdot q_j \).

If the conflict between sentences \( i \) and \( j \) is high (i.e., \( C_{ij} \) is small), then \( \exp(-\gamma (1 - C_{ij})) \) becomes small, which makes \( L_{ij} \) small. As a result, the negative term \( L_{ij}^2 \) in the determinant expression is reduced, increasing \( \det(L_{Y_g}) \). This indicates that the DPP discourages selecting both conflicting sentences, since doing so would reduce the determinant of the submatrix \( L_{Y_g} \).

On the other hand, if the conflict between \( i \) and \( j \) is low (i.e., \( C_{ij} \) is large), the term \( \exp(-\gamma (1 - C_{ij})) \) is close to 1, preserving the magnitude of \( L_{ij} \). In this case, the determinant \( \det(L_{Y_g}) \) reflects a higher probability of selecting both sentences \( i \) and \( j \) together.
\end{proof}

\begin{corollary}
The incorporation of conflict information via the weighted similarity matrix \( K_{\text{weighted}} = K_{\text{sim}} \circ \exp(-\gamma (1 - C)) \) ensures that conflicting sentences are less likely to be selected together in the DPP, while sentences with low conflict are more likely to be jointly selected. This outcome is achieved by modulating the off-diagonal terms \( L_{ij} \) in the kernel matrix \( L \), which directly affects the determinant and the subset selection probabilities.
\end{corollary}

\section{Prompts and Few-shot Exemplars}
\label{sec:prompt}
\begin{prompt}[title={Prompt \thetcbcounter: FEVER}]
You are an AI assistant specialized in fact-checking claims against a reliable knowledge base, such as Wikipedia. Your goal is to verify the accuracy of the given claims by finding supporting or refuting evidence. For each claim, determine if it is "SUPPORTS" or "REFUTES". Clearly cite the sources of your evidence, and ensure your explanations are concise and accurate. Your response should start with "Answer:". And your response is either "SUPPORTS" or "REFUTES".
\end{prompt}

\begin{prompt}[title={Prompt \thetcbcounter: TQA}]
You are an AI assistant specialized in answering questions based on a given text. Your goal is to read the provided text carefully and extract accurate answers to the questions based on the information within the text. Ensure your answers are concise and directly relevant to the questions asked. Do not use any information that is not present in the provided text. Your response should start with "Answer:".
\end{prompt}

\begin{prompt}[title={Prompt \thetcbcounter: HotpotQA}]
You are an AI assistant specialized in answering complex questions by performing multi-hop reasoning. Your goal is to gather and combine information from multiple sources to provide accurate answers to the given questions. Ensure your answers are clear, concise, and based on the information from the provided sources. Answer the given question directly and briefly, starting with "Answer:".
\end{prompt}

\begin{prompt}[title={Prompt \thetcbcounter: NQ}]
You are a helpful and meticulous assistant.\\ \\
You are provided with one or more contexts that may contain evidence to help you arrive at the answer. Answer the given question directly and briefly, starting with "Answer:".
\end{prompt}

\begin{prompt}[title={Prompt \thetcbcounter: FM2}]
You are an AI assistant specialized in fact-checking and natural language inference. Your task is to evaluate claims based on provided evidence from Wikipedia and determine whether the claim is true or false. Carefully examine the evidence, considering details such as timeframes, factual consistency, and logical coherence. \\ \\
For each claim: \\
- If the evidence supports the claim, respond with "True". \\
- If the evidence refutes the claim, respond with "False".\\ \\
Provide a brief explanation for your decision, clearly citing relevant parts of the evidence. Your response should start with "Answer:".
\end{prompt}

\begin{prompt}[title={Prompt \thetcbcounter: Prompt Format}]
[system prompt] \newline \newline
Context1: {[context c1]} \newline
Context2: {[context c2]} \newline
... \newline
Contextn: {[context cn]} \newline
Question: {[question q]} \newline \newline
Answer: {[answer a]} \newline \newline 
Context1: {[context c1]} \newline
Context2: {[context c2]} \newline
... \newline
Contextn: {[context cn]} \newline
Question: {[question q]} \newline \newline
Answer: {[answer a]} \newline \newline 
Question: {[question q]} \newline \newline
Answer: {[answer a]} \newline \newline 
Context1: {[context c1]} \newline
Context2: {[context c2]} \newline
... \newline
Contextn: {[context cn]} \newline
Question: {[question q]} 
\end{prompt}

\end{document}